\documentclass[letterpaper, 10 pt, conference]{ieeeconf}  

\IEEEoverridecommandlockouts                              

\usepackage{graphicx} 
\graphicspath{ {./images/} }
\usepackage[small]{caption}
\usepackage{subcaption}
\usepackage{amsmath} 
\usepackage{amssymb}  
            
\usepackage{stackengine}
\def\delequal{\mathrel{\ensurestackMath{\stackon[1pt]{=}{\scriptscriptstyle\Delta}}}}

\usepackage[utf8]{inputenc}
\usepackage[english]{babel}

\usepackage{amsthm}
\usepackage{soul}

\newtheorem{theorem}{Theorem}
\newtheorem{lemma}{Lemma}
\newtheorem{assumption}{Assumption}

\theoremstyle{definition}
\newtheorem{definition}{Definition}
\newtheorem{remark}{Remark}

\usepackage{kotex}

\usepackage{pifont}

\usepackage{xcolor}

\usepackage{tikz}

\newcommand\copyrighttext{%
  \footnotesize \textcopyright 2021 IEEE.  Personal use of this material is permitted.  Permission from IEEE must be obtained for all other uses, in any current or future media, including reprinting/republishing this material for advertising or promotional purposes, creating new collective works, for resale or redistribution to servers or lists, or reuse of any copyrighted component of this work in other works.}
\newcommand\copyrightnotice{%
\begin{tikzpicture}[remember picture,overlay]
\node[anchor=south,yshift=10pt] at (current page.south) {\fbox{\parbox{\dimexpr\textwidth-\fboxsep-\fboxrule\relax}{\copyrighttext}}};
\end{tikzpicture}%
}

\title{\LARGE \bf
Stability and Robustness Analysis of Plug-Pulling \\using an Aerial Manipulator
}

\author{Jeonghyun Byun$^{1}$, Dongjae Lee$^{1}$, Hoseong Seo$^{1}$, Inkyu Jang$^{1}$, Jeongjun Choi$^{1}$ and H. Jin Kim$^{1}$
\thanks{$^{1}$ The authors are with the Department of Aerospace Engineering, Seoul National University, Seoul, South Korea.
        {\tt\small \{quswjdgus97, ehdwo713, hosung37, leplusbon, lojol2327, hjinkim\}@snu.ac.kr}}%
}

\begin{document}

\maketitle
\copyrightnotice
\thispagestyle{empty}
\pagestyle{empty}

\begin{abstract}
In this paper, an autonomous aerial manipulation task of pulling a plug out of an electric socket is conducted, where maintaining the stability and robustness is challenging due to sudden disappearance of a large interaction force. The abrupt change in the dynamical model before and after the separation of the plug can cause destabilization or mission failure. To accomplish aerial plug-pulling, we employ the concept of hybrid automata to divide the task into three operative modes, i.e, wire-pulling, stabilizing, and free-flight. Also, a strategy for trajectory generation and a design of disturbance-observer-based controllers for each operative mode are presented. Furthermore, the theory of hybrid automata is used to prove the stability and robustness during the mode transition. We validate the proposed trajectory generation and control method by an actual wire-pulling experiment with a multirotor-based aerial manipulator.

\end{abstract}

\section{INTRODUCTION}
Aerial manipulation has been a growing research topic which aims to utilize the maneuverability of an aerial vehicle and the versatility of a robotic manipulator. Different from physically non-interacting passive tasks such as surveillance and remote sensing, an aerial manipulator can execute active tasks involving physical interaction such as grasping \cite{kim2013aerial,kim2019sampling}, valve turning \cite{korpela2014towards}, drawer opening \cite{kim2015operating}, contact inspection \cite{jimenez2015aerial}, transportation \cite{lee2016estimation,byun2020line}, and door opening \cite{lee2020aerial_1}. 

Despite various demonstrations of aerial manipulation tasks involving contact with environments, they usually involved a relatively low level of changes in the dynamic characteristics, and they rarely dealt with the transition during the physical interaction explicitly. In fact, there is a lack of research on the stability of aerial manipulation before and after the physical interaction. When a mode switch entails a significant change in the system response, neglecting it can lead to destabilization. Therefore, it is necessary to systematically analyze dynamical modes and design a robust controller to more realistically embrace the whole operation. 

As an example of aerial manipulation involving a drastic {change in dynamics}, this paper deals with the problem of pulling a plug from a socket using a multirotor equipped with a two degree-of-freedom (DOF) robotic arm. In this task, the large force exerted on the end-effector suddenly disappears after the plug is separated from the socket. For formal analysis of the stability, we first formulate hybrid automata \cite{goebel2009hybrid} which enclose all the dynamic models and the operative modes that have their own control laws different from another. Then, we design disturbance-observer (DOB)-based controllers \cite{back2009inner} for the respective operative modes and prove the stability and robustness of the formulated hybrid automata \cite{marconi2009control}.
\begin{figure}[t]
\centering
\includegraphics[width = 0.45\textwidth]{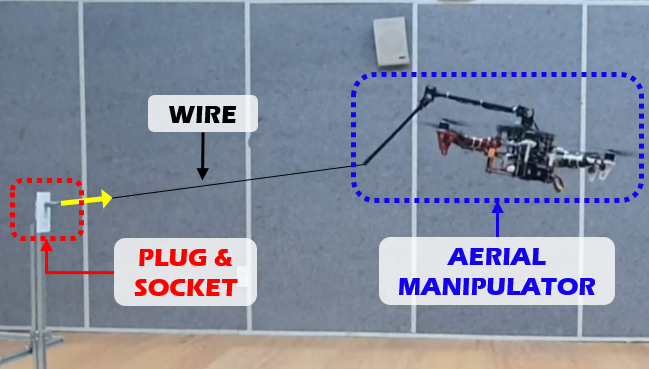}
\caption{An aerial manipulator, a multirotor equipped with a 2-DOF robotic arm, is pulling a plug from a socket. The plug is connected to a wire grabbed by the end effector of the aerial manipulator.} \label{fig: wire-pulling aerial manipulator}
\vspace{-0.5cm}
\end{figure}

\subsection{Related Works}
There have been several works which explain aerial vehicles using the concept of hybrid automata. In \cite{scholten2013interaction, darivianakis2014hybrid, Praveen2020inspection}, the contact task using an aerial vehicle is divided into two operative modes, i.e. docking and free-flight, and each mode are controlled by a different controller from one another. 
However, these works did not conduct an analysis on stability and robustness and the transition between those modes was not explicitly discussed. 

Some studies addressed the stability and robustness of an aerial vehicle involved in the physical interaction using the hybrid automata theory. In \cite{marconi2012control}, the stability of path-following control considering mode changes was investigated for a robust contact of a ducted-fan aerial vehicle on a vertical surface. In \cite{Cabecinhas2016robust}, the process of a multirotor landing on the slope was divided into several modes and the stability and robustness was proved in a similar way to \cite{marconi2012control}. However, in such settings, the effect of dynamic change can be reduced by slowly approaching the wall or landing site. Thus, there is no guarantee that such control methods can maintain the stability and robustness of the aerial manipulation involving an abrupt change such as plug-pulling. 

\subsection{Contributions}
To the best of the authors' knowledge, this is the first attempt to conduct a plug-pulling task using an aerial manipulator, which involves a significant mode change, and present a thorough analysis on the stability and robustness of the aerial manipulator using the hybrid automata theory. We propose a trajectory generation strategy and DOB control structure for each operative mode. Especially, for the situation of pulling the wire, we derive a dynamical model of the aerial manipulator constrained to the wire and the socket. In addition, we construct a DOB structure corresponding to the model of the plug-pulling aerial manipulator and prove the stability and robustness of the proposed controller.

\subsection{Outline}
In Section II, we briefly explain the concept of hybrid automata, describe notions utilized throughout the paper and introduce the aerial plug-pulling scenario. Section III formulates hybrid automata for the aerial manipulator conducting the plug task, and the trajectory generation and controller design is described in Section IV. Section V shows the stability and robustness analysis, and Section VI presents the experimental setup and results.

\section{PROBLEM SETUP}
\subsection{Preliminary: Hybrid Automata}\label{subsec:prelim}
The following elements define \textit{hybrid automata} \cite{goebel2009hybrid} with the state variable $x \in \mathbb{R}^{n_x}$ and the control input $u_x \in \mathbb{R}^{m_x}$.
\begin{itemize}
    \item \textit{Set of operative modes}, $\mathcal{M}$, contains names of the control modes. With respect to $\mathcal{M}$, we let $t_{0, \mu}$ and $t_{d,\mu}$ denote the time when the mode $\mu$ begins and the desired time to terminate the mode $\mu$, respectively. 
    \item  \textit{Domain mapping}, $\mathcal{D}: \mathcal{M} \rightrightarrows \mathbb{R}^{n_x} \times \mathbb{R}^{m_x}$,
     means the possible region where $x$ and $u_x$ can evolve while maintaining a specific mode $\mu$. It is expressed as $\mathcal{D}(\mu) = \mathcal{D}_{x}(\mu) \times \mathcal{D}_{u_x}(\mu)$.
    \item  \textit{Flow map}, $f$: $\mathcal{M} \times \mathbb{R}^{n_x} \times \mathbb{R}^{m_x} \rightarrow \mathbb{R}^{n_x/2}$, describes the dynamics in each operative mode $\mu$.
    \item \textit{Set of edges}, $\mathcal{E} \subset \mathcal{M} \times \mathcal{M}$, means all possible pairs of operative mode changes $(\mu_1, \mu_2)$.
    \item  \textit{Guard mapping}, $\mathcal{G}: \mathcal{E} \rightrightarrows \mathbb{R}^{n_x} \times \mathbb{R}^{m_x}$, describes the conditions where the transition from $\mu_1$ to $\mu_2$ occurs. It is represented as $\mathcal{G}(\{ \mu_1, \mu_2 \})$.
    \item \textit{Reset Map}, $\mathcal{R} : \mathcal{E} \times \mathbb{R}^{n_x} \times \mathbb{R}^{m_x} \rightarrow \mathbb{R}^{n_x}$, means the jump of the state variable $x$. It is expressed as $\mathcal{R}(\{ \mu_1,\mu_2 \},(x, u_{x}))$.
\end{itemize}

\subsection{Notations}
In this work, we use $0_{i j}$, $I_{i}$ and $e_3$ to denote the $i \times j$ zero matrix, the $i \times i$ identity matrix and $[0 \ 0 \ 1]^{\top}$. Also, we define $a_i$, $[a]$, dim$(a)$, $A_{i, j}$, $A_{i:j,k:l}$ and $\mathcal{B}_{\sigma}(a)$ as the $i$-th element of a column vector $a$, the $so(3)$ operator representing the cross product $[a]b = a \times b$, the dimension of $a$, the $(i,j)$-th element a matrix $A$, the block matrix of $A$ containing from $(i,k)$-th to $(j,l)$-th elements and the set $\{ c \in \mathbb{R}^{\textrm{dim}(a)} \ | \ \| c - a \| \leq \sigma \}$ where $\sigma$ is a constant positive number. The Kronecker product is expressed as $\otimes$.

As in Fig. \ref{fig: perching aerial manipulator with frames}, we denote the frame of the inertial coordinate, multirotor, 1$^{\textrm{st}}$, 2$^{\textrm{nd}}$ servo motor and the end-effector by $\{ I \}$, $\{ B \}$, $\{ \textit{1} \}$, $\{ \textit{2} \}$ and $\{ E \}$ respectively. 
\begin{figure}[t]
\centering
\includegraphics[width = 0.3\textwidth]{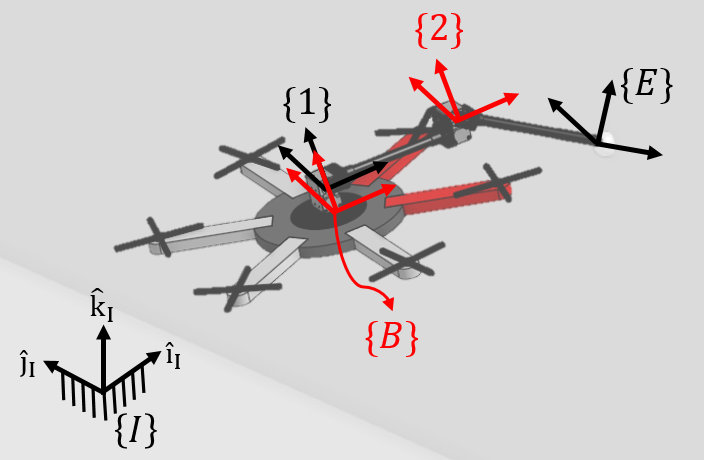}
\caption{The picture depicting an aerial manipulator perching on the wall connected by a spherical joint. The wall is aligned with $\hat{j}_I$-$\hat{k}_I$ plane, and the aerial manipulator pulls the plug out of the socket in -$\hat{i}_I$ direction.} \label{fig: perching aerial manipulator with frames}
\vspace{-0.5cm}
\end{figure}

To express the state of an aerial manipulator, we define the generalized coordinate $q$ as $[p_{I B}^{\top} \ \eta^{\top} \ \gamma^{\top}]^{\top}$ consisting of the position of the multirotor $p_{I B} \delequal [p_x \ p_y \ p_z]^{\top} \in \mathbb{R}^3$, Euler angles $\eta \delequal [\phi \ \theta \ \psi]^{\top} \in \mathbb{R}^3$ where $\phi$, $\theta$ and $\psi$ represent roll, pitch and yaw angles, and the angles of servo motors $\gamma \delequal [\gamma_1 \ \gamma_2]^{\top} \in \mathbb{R}^2$. Also, we let $\chi$, $r$, $x_q$ and $x_r$ denote $[p_{I B}^{\top} \ \eta^{\top}]^{\top}$, $[\eta^{\top} \ \gamma^{\top}]^{\top}$, $[q^{\top} \dot{q}^{\top}]^{\top}$ and $[r^{\top} \dot{r}^{\top}]^{\top}$. 
To represent inputs, we use $u_f$ and $u$ to denote $[T \ \tau_b^{\top} \  \tau_{\gamma}^{\top}]^{\top}$ and $[T \ \tau_b^{\top}]^{\top}$ where $T$, $\tau_b$ and $\tau_{\gamma}$ mean a total thrust, moments with respect to $\hat{i}_B$, $\hat{j}_B$ and $\hat{k}_B$ and torque inputs exerted on the servo motors. We let $Q \in \mathbb{R}^{3 \times 3}$ denote a matrix which satisfies $\omega^{B}_{I B} = Q\dot{\eta}$ and the scalar $g$  the gravitational acceleration. Moreover, we set $a_d$, $\hat{a}$ and $\bar{A}$ as a desired trajectory, an estimate of $a$ and 
the nominal value of $A$.

We use $m_b$, $m_1$ and $m_2$ to denote mass of the multirotor, the 1$^{\textrm{st}}$ and the 2$^{\textrm{nd}}$ servo motor while diagonal matrices $J_b$, $J_1$, $J_2$ in $\mathbb{R}^{3 \times 3}$ are the moments of inertia of the corresponding components. Additionally, $J_E \in \mathbb{R}^{3 \times 3}$ means the moment of inertia of the end-effector.  

\subsection{Scenario}
As in Fig. \ref{fig: perching aerial manipulator with frames}, the aerial manipulator tries to unplug in $-\hat{i}_I$ direction from the socket installed on the wall aligned with $\hat{j}_I$-$\hat{k}_I$ plane. After the plug is separated from the socket, the vehicle quickly stabilizes its attitude in a short time and maintains the hovering state.

\section{HYBRID AUTOMATA OF AERIAL PLUG-PULLING}
We construct the elements of the hybrid automata listed in Section \ref{subsec:prelim} for the aerial manipulator pulling the plug.
\subsection{Set of Operative Modes, $\mathcal{M}$ = $ \{$WP, ST, FF$\}$}
In $WP$ (\textit{wire-pulling}) mode, the aerial manipulator tries to unplug by pulling the wire in $-\hat{i}_I$ direction. In $ST$ (\textit{stabilizing}) mode, the vehicle quickly stabilizes its attitude immediately after the separation of the plug. In $FF$ (\textit{free-flight}) mode, the aerial manipulator returns to the original location and keeps the hovering state.

\subsection{Domain Mappings, $\mathcal{D}(WP)$, $\mathcal{D}(ST)$ and $\mathcal{D}(FF)$}
\begin{itemize}
    \item $\mathcal{D}(WP) \delequal
          \{ (x_r, u_f) \in  \mathbb{R}^{10} \times \mathbb{R}^6 \ | \ F_{E, 1} < F_{TH}\}$ where $F_{E}$ is {the interaction force} acting on the end-effector due to the friction between the plug and the socket and $F_{TH}$ is the force limit up to which the plug can resist from separating.
    \item $\mathcal{D}(FF \ or \ ST) \delequal \mathcal{V}_{F} - \mathcal{D}(WP)$ where $\mathcal{V}_{F}$ means an $\mathbb{R}^{16} \times \mathbb{R}^6$ space representing the flight envelope, i.e., all possible regions of the state and inputs for the flight experiment.
\end{itemize}

\subsection{Flow Maps, $f(WP,x_r,u_f)$ and $f(ST \ or \ FF,x_q,u_f)$} 
\subsubsection{$WP$ mode}
The derivation of $f(WP,x_r,u_f)$ is based on \cite{kim2013aerial}, but since the position of the end-effector is fixed, we newly derive it in the form of the Euler-Lagrange equation with $r = [\eta^{\top} \ \gamma^{\top}]^{\top}$ which fully describes the dynamics of the wire-pulling aerial manipulator.

First we express position and angular velocity as 
\begin{equation} \label{eq: WP - position and angular velocity values}
    \begin{split}
        p_{I B} &= p_{I E} -R_{I B}p_{B E}, \ p_{I1} = p_{I E} + R_{I B}p_{B1} - R_{I B}p_{B E}, \\
        p_{I2} &= p_{I E} + R_{I B}p_{B 2} - R_{I B}p_{B E}, \\
        \omega^{B}_{I B} &= Q \dot{\eta}, \quad \omega^{1}_{I 1} = R_{B 1}^{\top}(\omega^{B}_{I B} + \omega^{B}_{B1}),\\
        \omega^{2}_{I2} &= R_{b2}^{\top}(\omega^{b}_{I B} + \omega^{B}_{B2}), \quad \omega^{E}_{IE} = R_{B E}^{\top}(\omega^{B}_{I B} + \omega^{B}_{B E}).
    \end{split} 
\end{equation}
with the kinematic constraint $\dot{p}_{I E} = 0_{31}$
By substituting (\ref{eq: WP - position and angular velocity values}) into the derivation process presented in \cite{kim2013aerial}, an Euler-Lagrange equation of the wire-pulling aerial manipulator model can be derived. Then, the obtained equation of motion can be analyzed in the form of flow map as follows.
\begin{equation}\label{eq: dynamics_WP_full}
    \ddot{r} = f(WP,x_r,u_f) = M^{-1}_r(- C_r - K_r + J_{u r}^{T}u_{f} + \tau_{e,r})
\end{equation}
where $M_r \in \mathbb{R}^{5 \times 5}$, $C_r \in \mathbb{R}^{5}$, $K_r \in \mathbb{R}^{5}$,
\begin{equation*}
    J_{u r} \delequal \frac{\partial [\dot{p}^{B}_{I B} {}^{\top} \ {\omega^B_{I B}}^{\top} \ \dot{\gamma}^{\top} ]^{\top}}{\partial r} \in \mathbb{R}^{6 \times 5},
\end{equation*}
{and $\tau_{e, r}$ means the external disturbance applied to the wire-pulling aerial manipulator system.} Thus, the variable $x_r$ evolves in correspondence with (\ref{eq: dynamics_WP_full}).

\subsubsection{$ST$ and $FF$ modes}
The Euler-Lagrange equation for the multirotor equipped with the 2 DOF robotic arm is derived in \cite{kim2013aerial}. Then, the flow map $f(ST \ or \ FF,x_q,u_q)$ can be derived as follows.
\begin{multline} \label{eq:dynamics of FF_full}
    \ddot{q} = f(ST \ or \ FF,x_q,u_q) \\= M_q^{-1}(- C_q - K_q + J_{u q}^{T}u_f + \tau_{e,q})
\end{multline}
where $M_q \in \mathbb{R}^{8 \times 8}$, $C_q \in \mathbb{R}^{8}$, $K_q \in \mathbb{R}^{8}$, $J_{u q} \in \mathbb{R}^{6 \times 8}$, and $\tau_{e, q}$ means the external disturbance applied to the system in free flight.

\subsection{Set of Edges, $\mathcal{E} = \{ (WP,ST), (ST,FF) \}$}
There would be only two possible edges in our scenario because the transition from ST or FF mode to WP mode does not occur unless the plug is attached to the socket again. Also, a change from $FF$ to $ST$ mode is impossible because the $ST$ mode is primarily designed as an intermediate stage between the $WP$ and $FF$ modes.

\subsection{Guard Mappings, $\mathcal{G}(\{WP,ST\})$ and $\mathcal{G}(\{ST,FF\})$}
A transition from WP mode to ST mode occurs when the $\hat{i}_I$ element of $F_E$ exceeds $F_{TH}$. Therefore, the guard map $\mathcal{G}(\{WP,ST\})$ is defined as $\{ (x_r, u_f) \in \mathcal{D}(WP) \ | \ F_{E, 1} = F_{TH} \}$. Also, since ($\phi_d$, $\theta_d$) in the ST mode are user-defined values while they are computed from the user-defined values of ($x_d$, $y_d$) in the FF mode, an undesirable abrupt change in ($\phi_d$, $\theta_d$) would provoke a failure in attitude control. Therefore, the guard map $\mathcal{G}(\{ST,FF\})$ is defined as $\{(x_q,u_f) \in \mathcal{D}(ST) \ | \ \| \eta \| < \delta_{\eta}, \ t_{d, ST} \leq t \}$ where $\delta_{\eta}$ is defined as the threshold of $\eta$ for the mode change.

\subsection{Reset maps, $\mathcal{R}(\{ WP,ST \},(x_r, u_f))$ and $\mathcal{R}(\{ ST,FF \},(x_q, u_f))$}
If the operative mode changes from WP  to ST, there would be jumps in $x_q$ due to the sudden disappearance of the force exerted on the end-effector. However, since we cannot know the exact magnitude of the jumps in $x_q$, we denote it by $x_q$ which satisfies $\dot{p}_{I E}(x_q) \neq 0_{31}$. Then, the reset map $\mathcal{R}(\{ WP,ST \},(x_q, u_f))$ is expressed as $\{ x_q^{+} \in \mathcal{D}_{x_q^{+}}(ST) \ | \ \dot{p}_{I E}(x_q^{+}) \neq 0_{31} \ \textrm{where} \ x_r \in \mathcal{D}_{x_r}(WP) \}$. On the other hand, the change from ST  to FF does not entail any jump in $x_q$ because they evolve under the same dynamics. Therefore, the reset map $\mathcal{R}(\{ ST,FF \},(x_q, u_f))$ can be derived as $\{ x_q^{+} \in \mathcal{D}_{x_q^{+}}(FF) \ | \ x_q^{+} = x_q \ \textrm{where} \ x_q \in \mathcal{D}_{x_q}(ST) \}$.

\section{TRAJECTORY GENERATION AND CONTROLLER DESIGN}
\subsection{Trajectory Generation}
It is assumed that $\gamma$ and $\dot{\gamma}$ exactly follow $\gamma_d$ and $\dot{\gamma}_d$ respectively and the desired values that are not defined at each mode are set to be the same as the current values.
\subsubsection{WP mode}
In this mode, the aerial manipulator tries to {tilt its body with respect to $-\hat{j}_B$} in order to exercise a pulling force to the socket. Therefore, $\eta_d(t)$ is given as below,
\begin{equation}\label{eq: trajectory (WP)}
    \eta_d(t) =
    \begin{cases}
        [0 \ -\theta_{m}(\frac{t-t_{0,WP}}{t_{d, WP} - t_{0, WP}}) \ 0 ]^{\top}, & t_{0, WP} \leq t < t_{d, WP} \\
        0_{31}, & t_{d, WP} \leq t
    \end{cases}
\end{equation}
where $\theta_m$ means the maximum absolute value of the pitch angle. It prevents a sudden transition to $ST$ mode by gradually {tilting the vehicle's body}.
\subsubsection{ST mode}
This mode is proposed for compensating the overshoot invoked by the transition of the dynamical model and avoiding an abrupt change in $\phi_d(t)$ and $\theta_d(t)$. In order to simultaneously minimize the overshoot and make $\phi$ and $\theta$ close to zero, the time interval $[t_{0, ST}, t_{d, ST})$ needs to be reasonably small. Therefore, $p_{z, d}(t)$ and $\eta_d(t)$ are set as
\begin{equation*}
    p_{z, d}(t) = p_z(t_{0,ST})
\end{equation*}
\begin{equation}\label{eq: trajectory (ST)}
    \eta_d(t) =
    \begin{cases}
    c_2 t^2 + c_1 t + c_0,\ &t_{0, ST} \leq t < t_{d, ST}\\
    0_{31},\ & t_{d, ST} \leq t
    \end{cases}
\end{equation}
where coefficients $c_0$, $c_1$ and $c_2$ satisfy the conditions $\eta_d(t_{0, ST}) = \eta(t_{0, ST})$, $\dot{\eta}_d(t_{0, ST}) = \dot{\eta}(t_{0, ST})$ and $\eta_d(t_{d, ST}) = 0_{31}$.
\subsubsection{FF mode} 
In FF mode, $p_{I b}$ and $\psi$ are set to fly back to the original position as follows:
\begin{equation} \label{eq: trajectory (FF)}
    p_{I b, d}(t) = p_{I b}(t_{0, WP}), \ \psi_d(t) = \psi(t_{0, WP})
\end{equation}
\subsection{Nominal Model for Each Mode}
Servo motors are usually controlled by the given desired position, not torque. Therefore, the equations of motion that eliminate the term $\tau_{\gamma}$ are derived for each model.
\subsubsection{WP mode}
In (\ref{eq: dynamics_WP_full}), $J_{u r}$ is computed as
\begin{equation}
    J_{u r}^{\top} = \begin{bmatrix} J_{T \eta}^{\top} & Q^{\top} & 0_{3 2} \\ J_{T \gamma}^{\top} & 0_{2 3} & I_2 \end{bmatrix}, \ J_{T_{\eta}} \in \mathbb{R}^{1 \times 3}, \ J_{T_{\gamma}} \in \mathbb{R}^{1 \times 2}.
\end{equation}
Therefore, the model with respect to $\eta$ is obtained with {known values $T$, $\gamma$, $\dot{\gamma}$ and an observable value $\ddot{\gamma}$} as follows.
\begin{equation} \label{eq: actual model of WP mode}
    \ddot{\eta} = F_{\eta} + G_{\eta}\tau_b
\end{equation}
where
\begin{equation*}
\begin{split}
    F_{\eta} &\delequal M^{-1}_{\eta}\{ -C_{\eta} - K_{\eta} - M_{\eta \gamma}\ddot{\gamma} - J^{\top}_{T_{\eta}}T - \tau_{e,r} \} \\
    G_{\eta} &\delequal M^{-1}_{\eta}Q^{\top}
\end{split}
\end{equation*}
with block matrices $M_{\eta} = M_{r,1:3,1:3}$, $M_{\eta \gamma} = M_{r,1:3,4:5}$, $C_{\eta} = C_{r, 1:3,1}$ and $G_{\eta} = G_{r, 1:3, 1}$.
Then based on this, the nominal model for the WP mode can be obtained as follows.
\begin{equation} \label{eq: nominal dynamics of WP mode}
    \ddot{\eta} = \bar{F}_{\eta} + \bar{G}_{\eta}\tau_{b 0}
\end{equation}
where
\begin{equation} \label{eq: definition of F_eta and G_eta}
    \begin{split}
        \bar{F}_{\eta} &\delequal \bar{M}^{-1}_{\eta}\{ -\bar{C}_{\eta} - \bar{K}_{\eta} - \bar{M}_{\eta \gamma}\ddot{\hat{\gamma}} - \bar{J}^{\top}_{T_{\eta}}T \} \\
    \bar{G}_{\eta} &\delequal \bar{J}_b^{-1} Q^{\top}
    \end{split}
\end{equation}
with the nominal input $u_0 = [T \ \tau_{b 0}^{\top}]^{\top}$. The total thrust $T$ is calculated in the DOB controller introduced in \cite{lee2020aerial_2}.

\subsubsection{ST and FF mode}
The nominal model for ST and FF mode is derived in \cite{lee2020aerial_2} as below.
\begin{equation} \label{eq: actual and nominal model for FF}
        \ddot{q}_u = \bar{G}_u \Phi_0, \qquad \ddot{q}_f = \bar{F}_f + \bar{G}_f u_0
\end{equation}
where $q_u$ and $q_f$ mean the center of mass of $[p_x \ p_y]^{\top}$ and $[p_z \ \eta^{\top}]^{\top}$ respectively. The other notations are defined in \cite{lee2020aerial_2}.

\subsection{Controller Design}
\subsubsection{WP mode}
If we design the nominal input $\tau_{b 0}$ to make the solution of (\ref{eq: nominal dynamics of WP mode}) adequately follows $\eta_d$, the compensation of model discrepancy,  $\Delta_{\eta} \delequal (F_{\eta} - \bar{F}_{\eta}) + (G_{\eta} \tau_b - \bar{G}_{\eta} \tau_{b 0})$, is conducted by the DOB structure presented in \cite{back2009inner}. The overall diagram is shown in Fig. {\ref{fig: WPmode_DOB_structure}} and the detailed DOB control law is formulated as below.
\begin{figure}[t]
\centering
\includegraphics[width = 0.42\textwidth]{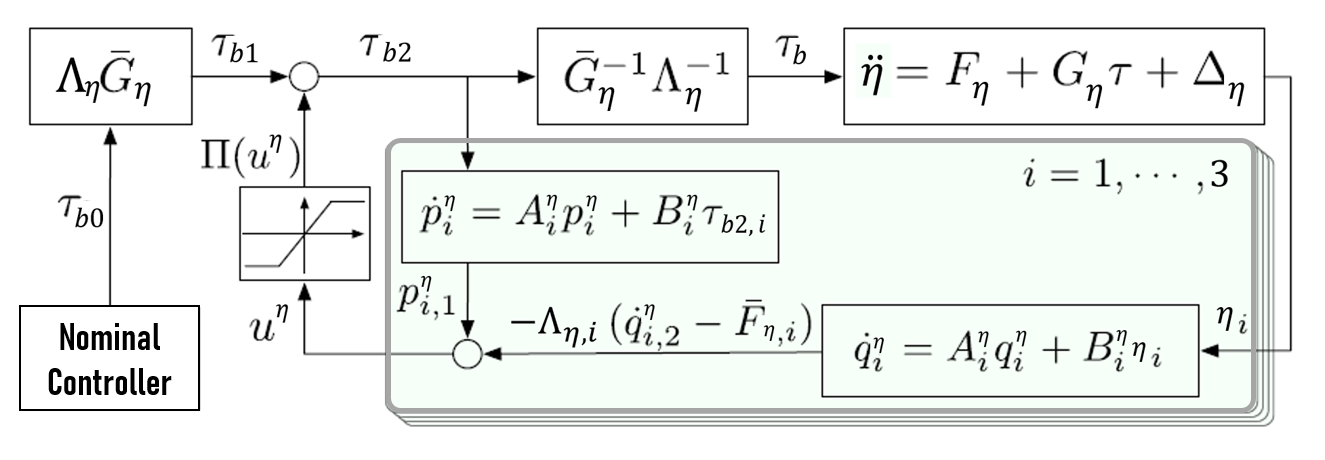}
\caption{The DOB structure of WP mode for compensating the model discrepancy $\Delta_{\eta}$ between (\ref{eq: actual model of WP mode}) and (\ref{eq: nominal dynamics of WP mode}).} \label{fig: WPmode_DOB_structure}
\vspace{-0.5cm}
\end{figure}
\begin{equation} \label{nominal controller for wire-pulling mode}
\begin{split}
    \dot{q}^{\eta}_{i} &= A_{\eta, i}q^{\eta}_{i} + B_{\eta, i}\eta_{i}, \quad \dot{p}^{\eta}_{i} = A_{\eta, i}p^{\eta}_{i} + B_{\eta, i}\tau_{b2, i} \\
    u^{\eta}_i &= p^{\eta}_{i, 1} - \sum_{j = 1}^{3}\Lambda_{\eta, i, j}(\dot{q}^{\eta}_{j, 2}-\bar{F}_{\eta, j}), \quad \tau_{b1} = \Lambda_{\eta}\bar{G}_{\eta}\tau_{b0} \\ 
    \tau_{b2} &= \Lambda_{\eta}\bar{G}_{\eta}\tau_{b 0} + \Pi_{\eta}(u^{\eta}), \ \tau_b = \tau_{b 0} + (\Lambda_{\eta}\bar{G}_{\eta})^{-1}\Pi_{\eta}(u^{\eta})
\end{split}
\end{equation}
where $q^{\eta} = [q^{\eta}_{1,1} \ q^{\eta}_{1,2} \ ...\ q^{\eta}_{3,2}]^{\top}$, $p^{\eta} = [p^{\eta}_{1,1} \ p^{\eta}_{1,2} \ ...\ p^{\eta}_{3,2}]^{\top}$, $u^{\eta} = [u^{\eta}_1 \ u^{\eta}_2 \ u^{\eta}_3]^{\top}$ and
\begin{equation*}
\begin{split}
    A^{\eta}_i &= 
    \begin{bmatrix}
        0 & 1 \\
        -a^{\eta}_{i, 0}/{\epsilon}^2_{\eta} & -a^{\eta}_{i, 1}/\epsilon_{\eta}
    \end{bmatrix}, \quad B^{\eta}_i = 
    \begin{bmatrix}
         0 \\
        a^{\eta}_{i, 0}/\epsilon_{\eta}^2
    \end{bmatrix}\\
\Lambda_{\eta} &= \bar{J}_b^{(1/2)} Q \in \mathbb{R}^{3 \times 3}
\end{split}
\end{equation*}
with the positive constants $a^{\eta}_{i, 0}$, $a^{\eta}_{i, 1}$ and the small positive constant $\epsilon_{\eta}$. In this DOB structure, {we use a saturation function $\Pi_{\eta}$ defined with the conditions below.}
\begin{itemize}
    \item $\Pi_{\eta}$: $\mathbb{R}^3 \Rightarrow \mathbb{R}^3$ is a globally bounded $\mathcal{C}^1$ function.
    \item \label{equality condition for saturation function} $\Pi_{\eta}(u^{\eta}) = u^{\eta}$ for $\forall u^{\eta} \in S_{u^{\eta}}$ where $S_{u^{\eta}} \delequal \{ u^{\eta} \in \mathbb{R}^{3 \times 1} \ | \ u^{\eta} = \Lambda_{\eta}\bar{G}_{\eta}G_{\eta}^{-1}(\bar{F}_{\eta} - F_{\eta} + (\bar{G}_{\eta} - G_{\eta})\tau_{b 0} - \Delta_{\eta}) \}$.
    \item $\| \partial \Pi_{\eta}(u^{\eta}) / \partial u^{\eta} \| \leq 1$ for $\forall u^{\eta} \in \mathbb{R}^{3 \times 1}$
\end{itemize}
From the conditions above, the quasi-steady state range of $u^{\eta}$ satisfies the equation $u^{\eta}_i = p^{\eta}_{i, 1} - \sum_{j = 1}^{3}\Lambda_{\eta, i, j}(\dot{q}^{\eta}_{j, 2}-\bar{F}_{\eta, j})$ while avoiding the saturation \cite{back2009inner}.

\subsubsection{ST mode}
For the ST mode, we only apply the DOB structure for the fully-actuated system introduced in \cite{kim2017robust}.

\subsubsection{FF mode}
We will utilize the same controller presented in \cite{lee2020aerial_2} for the FF mode.

\section{STABILITY AND ROBUSTNESS ANALYSIS}
In this section, an analysis of the stability and robustness will be presented. During the analysis, the term \textit{maneuver} which means $[x_{q}^{\top} u_f^{\top}]^{\top}$ in a particular mode will be used. $v_{\mu}$ and $\bar{v}_{\mu}^N$ mean the solution from the actual flow map and the nominal flow map. Additionally, tr$v_{\mu}$ is defined as a set including all values of $v_{\mu}$ in the given time interval.

\subsection{Preliminary Definitions for the Analysis on Hybrid Automata}
\begin{definition}[$\sigma$-robust $\mu_1$-single maneuver in $[t_{0, \mu_1}, t_1$)]\label{def: robust single maneuver} 
For $\sigma > 0$, $\mu_1 \in \mathcal{M}$ and $0 < t_1$, a maneuver $\bar{v}^N_{\mu_1} \in [t_{0, \mu_1}, t_1)_{\mu_1}$ satisfies
\begin{equation}
    \textrm{tr} \ \bar{v}^N_{\mu_1} \bigcap \big( \bigcup_{\{  \mu_1,  \mu' \} \in \mathcal{E}} \mathcal{G}(\{ \mu_1, \mu' \}) + \mathcal{B}_{\sigma} \big) = \varnothing.
\end{equation}  
\end{definition} 
Here, $[t_{0, \mu}, t_1)_{\mu}$ means the time interval where $\bar{v}^N_{WP}$ evolves in the mode $\mu$ within $[t_{0, \mu}, t_1)$.

\begin{definition}[$\sigma$-robust $\mu_{1} \mapsto \mu_{2}$ approach maneuver in $[t_{0, \mu_1},t_1 $)]\label{def: robust approach maneuver} 
For $\sigma > 0$, $\mu_1, \mu_2 \in \mathcal{M}$ and $t_{0, \mu_1} \leq t_1$, a maneuver $\bar{v}^N_{\mu_1} \in [t_{0, \mu_1}, T]_{\mu_1}$ satisfies
\begin{itemize}
    \item $\textrm{tr} \ \bar{v}^N_{\mu_1} \bigcap \big( \bigcup_{\{ \mu_1, \mu_2' \} \in \mathcal{E}-\{ \mu_1, \mu_2 \}} \mathcal{G}(\{ \mu_1, \mu_2' \}) + \mathcal{B}_{\sigma} \big) = \varnothing$ and $\mathcal{B}_{\sigma}(\bar{v}^N_{\mu_1}) \subset \mathcal{G}(\{ \mu_1, \mu_2 \})$
    \item Let $S^{\mu_1}$ and $S^{\mu_1 \rightarrow \mu_2}$ be compact sets defined as $S^{\mu_1} \delequal (\textrm{tr} \bar{v}^N_{\mu_1} + \mathcal{B}_{\sigma}) \cap \mathcal{G}(\{ \mu_1, \mu_2 \})$ and
    \begin{multline}
        S^{\mu_1 \rightarrow \mu_2} \delequal \{x_q^{+} \in \mathcal{D}_{x_q}(\mu_2) \ | \ x_q^{+} \in \\ \mathcal{R}(\{ \mu_1, \mu_2 \}, (x_q,u_f) \ \textrm{for some} \ (x_q,u_f) \in S^{\mu_1} \}
    \end{multline}
    Then, for any $x_q^{+} \in S^{\mu_1 \rightarrow \mu_2}$ there exists $u_f^{+}$ such that 
    \begin{equation}
        (x_q^{+}, u_f^{+}) \notin \bigcup_{\{ \mu_2, \mu_2' \} \in \mathcal{E}} \mathcal{G}(\{ \mu_2, \mu_2' \}) + \mathcal{B}_{\sigma}.
    \end{equation} 
\end{itemize}
\end{definition}

\begin{definition}[$(\sigma,\delta_{\sigma})$-robust $S^{\mu_1 \mapsto \mu_2}$ coverage set, $\mathcal{C}^{\mu_1 \mapsto \mu_2}_{\delta_\sigma}$] \label{def: C_delta_sigma}
A set of $N_{\delta_{\sigma}}$ elements $x_{q, 1},...,x_{q, N_{\delta_{\sigma}}} \in \mathcal{D}_{x_q}(\mu_2)$ which satisfies
\begin{equation}
    S^{\mu_1 \mapsto \mu_2} \subset \bigcup_{j \in \{ 0, ..., N_{\delta_{\sigma}} \}} \mathcal{B}_{\delta_{\sigma}}(x_{q, j}).
\end{equation}
\end{definition}

\begin{definition} [$(\sigma,\delta_{\sigma})$-robust $\mu_{1} \mapsto \mu_{2}$ transition maneuver] \label{def: definition of robust transition}
For $\sigma > 0$, $\delta_{\sigma} > 0$, $\mu_1$, $\mu_2 \in \mathcal{M}$ and $t_{0, \mu_1} < t_1$, $\bar{v}^N_{\mu_1}$ which is a union of a $\sigma$-robust $\mu_{1} \mapsto \mu_{2}$ approach maneuver before the switching time $T \in (t_{0, \mu_1}, t_1)$ and of a set of $N_{\delta_{\sigma}}$ $\sigma$-robust single maneuvers after transition with the property $v_{\mu_2, j} \in \mathcal{C}^{\mu_1 \mapsto \mu_2}_{\delta_\sigma}, \forall j \in \{ 1, ..., N_{\delta_{\sigma}}\}$. 
\end{definition}

\subsection{Robustness of the Nominal maneuver}
With the assumption that $x_q$ of the nominal maneuver $\bar{v}^N_{\mu}$ adequately follows {$x_{q, d}$} defined through (\ref{eq: trajectory (WP)}) -- (\ref{eq: trajectory (FF)}), we analyze the characteristics of the nominal maneuver. 
\subsubsection{$(\sigma, \delta_{\sigma})$-robust $WP \mapsto ST$ transition maneuver}
With the assumption that the aerial manipulator is in a quasi-equilibrium state while perching on the wall, the pulling force $F_{E, 1}$ is equal to $-T\sin \theta$ \cite{kim2015operating}. Thus, $F_{E, 1}$ increases with the gradually decreasing trajectory of $\theta$ as generated in (\ref{eq: trajectory (WP)}). As a result, $\bar{v}^N_{WP}$ does not reach $\mathcal{G}\{ (WP, ST) \} + \mathcal{B}_{\sigma}$ before $F_{E, 1}$ closely approaches $F_{TH}$. Moreover, as depicted in Fig. \ref{fig:diagram_for_WP_ST_1}, there is no possibility of a transition from $WP$ to $FF$ mode because $\mathcal{G}\{ (WP, FF) \}$ is defined as $\varnothing$. Therefore, the maneuver $\bar{v}^N_{WP}$ is proved to be a $\sigma$-robust $WP$-single maneuver in $[t_{0, WP}, t_{0, ST})$. 

$\bar{v}^N_{WP}$ can also become a $\sigma$-robust $WP \mapsto ST$ approach maneuver when there exists a time instant that a change from $WP$ to $ST$ occurs in a finite time. Thanks to the relation that $F_{E,1}$ equals to $-T\sin\theta$, we can easily find a sufficient condition $0 < F_{TH} < T_{m}\sin \theta_m$ where $T_{m}$ means the maximum value of $T$. The inequality above infers that $F_{E, 1}$ can reach $F_{T H}$ with the given trajectory of $\theta$. {Additionally, since the direct change from $WP$ to $FF$ mode never occurs as mentioned above}, the claim that $\bar{v}^N_{WP}$ is a robust approach maneuver is proved.

After the transition from $WP$ to $ST$ mode, the reset map of $\bar{v}^N_{WP}$ is uncertain. However, we can avoid the situation where the value of $\mathcal{R}(\{ (WP, ST) \}, (x_r(t_{0, ST}), u_f(t_{0, ST})))$ becomes an element of $\mathcal{D}(FF)$ the switch from $ST$ to $FF$ does not occur before the time reaches the value of $t_{d, ST}$ by the condition $t_{d, ST} < t$ in $\mathcal{G}\{ (ST, FF) \}$. Therefore, a maneuver $\bar{v}^N_{WP}$ turns out to be a $\sigma$-robust $ST$-single maneuver in $[t_{0, WP}, t_{0, ST})$ after the mode transition. From the analyses above, the union of tr$\bar{v}^N_{W P}$ and tr$\bar{v}^N_{ST}$ is proved to be a $(\sigma, \delta_{\sigma})$-robust $WP \mapsto ST$ transition maneuver.
\begin{figure}[t] \centering
\begin{subfigure}{0.45\textwidth} \centering
\includegraphics[width=0.55\linewidth]{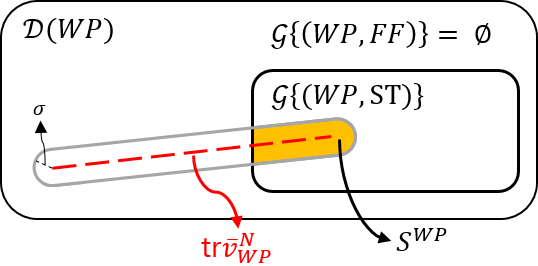} 
\caption{$\sigma$-robust $WP \mapsto ST$ approach maneuver}
\label{fig:diagram_for_WP_ST_1}
\end{subfigure}
\begin{subfigure}{0.45\textwidth} \centering
\includegraphics[width=0.7\linewidth]{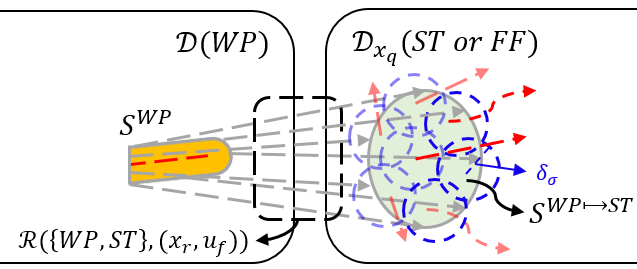}
\caption{$(\sigma, \delta_{\sigma})$-robust $WP \mapsto ST$ transition maneuver.}
\label{fig:diagram_for_WP_ST_2}
\end{subfigure}
\caption{Venn diagram demonstrating hybrid automata in $WP$ and $ST$ modes. (a) A red dashed line means the trace of the nominal maneuver in $WP$ mode. A yellow region describes the set of possible maneuvers which can provoke a change from $WP$ to $ST$ mode. (b) A green region expresses the set of reset maps from the maneuvers initiating from the yellow region. Blue dashed circles are centered in $N_{\delta_{\sigma}}$ elements in $C^{WP \mapsto ST}_{\delta_{\sigma}}$ and have the radius $\delta_{\sigma}$. Red dashed arrows describe $N_{\delta_{\sigma}}$ maneuvers starting from $x_{q, j}$ where $j = 1, ..., N_{\delta_{\sigma}}$. Following Definition \ref{def: C_delta_sigma}, all the blue dashed circles cover the green region.}
\label{fig:diagram_for_WP_ST}
\vspace{-0.5cm}
\end{figure}

\subsubsection{$(\sigma, \delta_{\sigma})$-robust $ST \mapsto FF$ transition maneuver}
As proved in the previous section, the maneuver $\bar{v}^N_{ST}$ is a $\sigma$-robust $ST$-single maneuver in $[t_{0, ST}, t_{0, FF})$. Also, there must be a change from $ST$ to $FF$ mode in a finite time since $\eta_d$ defined in $ST$ mode reaches $0_{31}$ when $t$ reaches $t_{d, ST}$. Also, since the transition from ST to WP mode is impossible, $\bar{v}^N_{ST}$ also turns out to be a $\sigma$-robust $ST \mapsto FF$ approach maneuver in $[t_{0, ST}, t_{0, FF})$.

Accordingly, the property that $(FF, ST)$ and $(FF, WP)$ are not the elements of $\mathcal{E}$ also guarantees that a maneuver $\bar{v}^N_{FF}$ is a $\sigma$-robust $FF$-single maneuver in $[t_{0, FF}, t_{d, FF})$. Thus, it is proved   that the union of tr$\bar{v}^N_{ST}$ and tr$\bar{v}^N_{FF}$ becomes a $(\sigma, \delta_{\sigma})$-robust $ST \mapsto FF$ transition maneuver.

\subsection{Analysis on Stability and Robustness at Each Mode}
\subsubsection{Wire-pulling mode}
Prior to formulating a theorem about the stability and robustness of the $WP$ mode, there need some remarks and assumptions as below.
\begin{remark}
$\eta_d$, $\dot{\eta}_d$, $\ddot{\eta}_d$ are continuous and bounded in $\mathcal{C}^2$. In addition, let $\eta_d \in S_{\eta_d}$ where $S_{\eta_d}$ is a known compact set. 
\end{remark}
\begin{assumption}
Let $\bar{\eta}^N(t)$ be the nominal solution of (\ref{eq: nominal dynamics of WP mode}). For given $\eta_d$, the solution $\bar{\eta}^N(t)$ evolves in a bounded set $\mathcal{U}_{\eta}$ if the initial condition $\bar{\eta}^N(0)$ is in a compact set $S_{\eta} \subset \mathcal{U}_{\eta}$, and $\bar{\eta}^N(t) - \eta_d(t)$ initiated in $S_{\eta}$ is locally asymptotically stable.
\end{assumption}
\begin{assumption} \label{assumption on disturbance in wire-pulling mode}
According to \cite{kim2017robust}, the aerodynamic effects, such as drag or buoyancy forces, are negligible in a near hovering condition due to the small size of the multirotors. Also, frictional torque and force applied on the end-effector are at least $\mathcal{C}^2$ and bounded in $\mathcal{U}_{\eta}$ because they are functions of $r$ and $\dot{r}$.
\end{assumption}
\begin{remark} \label{remark: C^2 and boundedness of F,G}
The terms $M_r$, $\bar{M}_{r}$, $M_{\eta \gamma}\ddot{\gamma}$, $\bar{M}_{\eta \gamma}\ddot{\hat{\gamma}}$, $C_r$, $\bar{C}_{r}$, $G_r$, $\bar{G}_{r}$, $J_{u r}$ and $\bar{J}_{u r}$ are vectors and matrices which consist of $r$, $\dot{r}$ and $\ddot{r}$. Thus, these terms are at least $\mathcal{C}^2$ and bounded in $\mathcal{U}_{\eta}$. Moreover, by the assumption \ref{assumption on disturbance in wire-pulling mode}, $\tau_{e,r}$  is also a vector that is $\mathcal{C}^2$ and bounded in $\mathcal{U}_{\eta}$. Since $F_{\eta}$, $\bar{F}_{\eta}$, $G_{\eta}$ and $\bar{G}_{\eta}$ have the terms described above, vectors and matrices shown in (\ref{eq: nominal dynamics of WP mode}) are all at least $\mathcal{C}^2$ and bounded in $\mathcal{U}_{\eta}$.
\end{remark}
Then finally, from the assumptions and remarks stated above, we can formulate a theorem on the relationship similar to the theorem introduced in \cite{lee2020aerial_2} and prove it.
\begin{theorem} \label{thm: stability analysis on wire-pulling mode}
Let $S^{\eta}_{q p}$ be a compact set for the initial condition $[q^{\eta}(0)^{\top} \ p^{\eta}(0)^{\top}]^{\top}$, and $\bar{S}_{\eta}$ be a compact set smaller than $S_{\eta}$. For a given $\sigma > 0$, there exists $\epsilon^{*}_{\eta}$ such that, for each $0 < \epsilon_{\eta} < \epsilon^{*}_{\eta}$, the solution of (\ref{eq: nominal dynamics of WP mode}), $[\bar{\eta}^N(t)^{\top} \dot{\eta}^{N}(t)^{\top} \bar{\tau}^N_{b 0}(t)^{\top}]^{\top}$, and that of (\ref{eq: actual model of WP mode}), $[\eta(t)^{\top} \dot{\eta}^{N}(t)^{\top} \tau_{b 0}(t)^{\top}]^{\top}$, initiated at $[\eta(0)^{\top} \dot{\eta}(t)^{\top} \tau_{b 0}(t)^{\top} q^{\eta}(0)^{\top} p^{\eta}(0)^{\top}]^{\top} \in \bar{S}_{\eta} \times S^{\eta}_{q p}$ satisfies
\begin{multline*}
    \| [\eta(t)^{\top} \dot{\eta}(t)^{\top} \tau_{b 0}(t)^{\top} ] \\- [\bar{\eta}^{N}(t)^{\top} \dot{\eta}^{N}(t)^{\top} \bar{\tau}^{N}_{b 0}(t)^{\top} ] \| \leq \sigma, \ \forall t \geq 0
\end{multline*}
if we set the initial conditions for both dynamic models identical, $[\bar{\eta}^N(0)^{\top} \dot{\eta}^{N}(0)^{\top} \bar{\tau}^N_{b 0}(0)^{\top}]^{\top} = [\eta(0)^{\top} \dot{\eta}(0)^{\top} \tau_{b 0}(0)^{\top}]^{\top}$.
\end{theorem}
\begin{proof}
Proof of this theorem will follow the procedure introduced in \cite{kim2017robust}. $\hfill \blacksquare$
\end{proof}

\begin{lemma} \label{lemma for standard singular form}
With fast variables $\xi^{\eta} \delequal [\xi^{\eta T}_1 \ \xi^{\eta T}_2 \ \xi^{\eta T}_3]^{\top}$ and $\ \zeta^{\eta} \delequal [\zeta^{\eta T}_1 \ \zeta^{\eta T}_2 \ \zeta^{\eta T}_3]^{\top} \in \mathbb{R}^{6 \times 1}$ defined as 
\begin{equation} \label{eq: standard singular form of wire-pulling mode}
    \begin{split}
    \xi^{\eta}_i &= 
    \begin{bmatrix}
        \xi^{\eta}_{i, 1} \\
        \xi^{\eta}_{i, 2}
    \end{bmatrix} = 
    \begin{bmatrix}
        \frac{1}{\epsilon_{\eta}} q^{\eta}_{i, 1} + \frac{a^{\eta}_{i, 1}}{a^{\eta}_{i, 0}}q^{\eta}_{i, 2} - \frac{1}{\epsilon_{\eta}}\eta_{i}\\
        q^{\eta}_{i, 2} - \dot{\eta}_i
    \end{bmatrix} \in \mathbb{R}^{2 \times 1} \\
    \zeta^{\eta}_i &= 
    \begin{bmatrix}
        \zeta^{\eta}_{i, 1} \\
        \zeta^{\eta}_{i, 2}
    \end{bmatrix} = 
    \begin{bmatrix}
        p^{\eta}_{i, 1} - \Lambda_{i}\dot{q}^{\eta}_{i, 2}\\
        \epsilon_{\eta}(\dot{p}^{\eta}_{i, 1} - \ddot{q}^{\eta}_{i, 2})
    \end{bmatrix} \in \mathbb{R}^{2 \times 1},
    \end{split}
\end{equation}
the closed-loop system of the actual model describing wire-pulling mode can be rearranged in the standard singular perturbation form as follows.
\begin{equation}
    \begin{split}
        \ddot{\eta} &= F_{\eta} + G_{\eta}(\tau_{b 0} + (\Lambda_{\eta}\bar{G}_{\eta})^{-1}\Pi(u^{\eta}))+\Delta_{\eta} \\
        {\epsilon}_{\eta} \dot{\xi}^{\eta}_{i} &= A^{\eta}_{\xi, i}\xi^{\eta}_{i}-\epsilon_{\eta} B^{\eta}_2(F_{\eta, i} + G_{\eta, i}\tau_b + \Delta_{\eta, i}) \\ 
        {\epsilon}_{\eta} \dot{\zeta}^{\eta}_{i} &= A^{\eta}_{\zeta, i}\zeta^{\eta}_{i}+ B^{\eta}_2 a^{\eta}_{i, 0} (\tau_{b 2, i} \\ & \qquad \quad \quad - \sum_{j = 1}^{3} \Lambda_{\eta, i, j}(F_{\eta, j} + G_{\eta, j}\tau_b + \Delta_{\eta, j}))
    \end{split}
\end{equation}
where $B^{\eta}_2 \delequal [0 \ 1]^{\top}$ and
\begin{equation*}
    A^{\eta}_{\xi, i} \delequal
    \begin{bmatrix}
        -a^{\eta}_{i, 1} & 1 \\
        -a^{\eta}_{i, 0} & 0
    \end{bmatrix}, \quad A^{\eta}_{\zeta, i} \delequal
    \begin{bmatrix}
         0 & 1 \\
        -a^{\eta}_{i, 0} & -a^{\eta}_{i, 1}
    \end{bmatrix}.
\end{equation*}
\end{lemma}
\begin{proof}
The derivation of Lemma \ref{lemma for standard singular form} is based on \cite[Appendix A]{bang2009robust}. $\hfill \blacksquare$ 
\end{proof} 

To obtain the reduced model of (\ref{eq: actual model of WP mode}), the quasi-states of ($\xi^{\eta}$, $\zeta^{\eta}$) are obtained as follows by setting $\epsilon_{\eta} = 0$.
\begin{equation} \label{eq: quasi-states of xi and zeta}
    \begin{split}
        \xi^{\eta, *} &= 0_{6 1} \\  
        \zeta^{\eta, *}_{[1]} &= \tau_{b 2} - \Lambda_{\eta}(F_{\eta} + G_{\eta}\tau_b + \Delta_{\eta}), \ \zeta^{\eta, *}_{[2]} = 0_{3 1}
    \end{split}
\end{equation}
where $\zeta^{\eta, *}_{[1]} \delequal [\zeta^{\eta, *}_{1, 1} \ \zeta^{\eta, *}_{2, 1} \ \zeta^{\eta, *}_{3, 1}]^{\top}$ and $\zeta^{\eta, *}_{[2]} \delequal [\zeta^{\eta, *}_{1, 2} \ \zeta^{\eta, *}_{2, 2} \ \zeta^{\eta, *}_{3, 2}]^{\top}$.

Since $u^{\eta}$ turns out to be equal to $\zeta^{\eta}_{[1]} + \Lambda_{\eta} \bar{F}_{\eta}$ from (\ref{nominal controller for wire-pulling mode}) and (\ref{eq: standard singular form of wire-pulling mode}), an equation on $u^{\eta *}$ can be formulated as
\begin{multline} \label{eq: equation for ustar}
    u^{\eta *} - \Lambda_{\eta}((\bar{F}_{\eta} - F_{\eta}) + (\bar{G}_{\eta} - G_{\eta})\tau_{b 0} - \Delta_{\eta}) \\- (I_{3} - \Lambda_{\eta}G_{\eta}\bar{G}_{\eta}^{-1}\Lambda_{\eta}^{-1})\Pi_{\eta}(u^{\eta *}) = 0_{3 1}.
\end{multline}
Moreover, since $u^{\eta *} \in S_{u^\eta}$, $\Pi_{\eta}(u^{\eta *})$ equals $u^{\eta *}$. Therefore, an explicit expression for $u^{\eta *}$ is obtained as follows.
\begin{equation} \label{eq: solution for ustar}
    u^{\eta *} = \Lambda_{\eta} \bar{G}_{\eta}G_{\eta}^{-1}(\bar{F}_{\eta} - F_{\eta} + (\bar{G}_{\eta} - G_{\eta})\tau_{b 0} - \Delta_{\eta})
\end{equation}

The next step is to show that the only solution for (\ref{eq: equation for ustar}) is (\ref{eq: solution for ustar}). Thus, we declare a function $\Gamma_{\eta}(\delta)$ which replaces $u^{\eta *}$ in (\ref{eq: equation for ustar}) with $u^{\eta *} + \delta$ as below.
\begin{equation*}
    \begin{split}
        \Gamma_{\eta}(\delta) \delequal& u^{\eta *} + \delta - \Lambda_{\eta}((\bar{F}_{\eta} - F_{\eta}) + (\bar{G}_{\eta} - G_{\eta})\tau_{b 0} - \Delta_{\eta}) \\&- (I_{3} - \Lambda_{\eta}G_{\eta}\bar{G}_{\eta}^{-1}\Lambda_{\eta}^{-1})\Pi_{\eta}(u^{\eta *}+\delta) \\ =& \delta + (\Lambda_{\eta}G_{\eta}\bar{G}_{\eta}^{-1}\Lambda_{\eta}^{-1} - I_3)(\Pi_{\eta}(u^{\eta *} + \delta ) - \Pi_{\eta}(u^{\eta *})).
    \end{split}
\end{equation*}

\begin{lemma} \label{lemma for the sector condition of Gamma}
With the assumption that $\eta \in \mathcal{U}_{\eta}$, $\Gamma_{\eta}(\delta)$ belongs to the sector [$1-\kappa, 1+ \kappa$] with $0 < \kappa < 1$. Then accordingly, $\delta = 0_{3 1}$ is the unique solution of $\Gamma_{\eta}(\delta) = 0_{3 1}$.
\end{lemma}

\begin{proof}
From the definition of $\Gamma_{\eta}(\delta)$, the following inequality is derived.
\begin{equation*}
    | \Gamma_{\eta}(\delta) - \delta | \leq \| I_{3} - \Lambda_{\eta}G_{\eta}\bar{G}_{\eta}^{-1}\Lambda_{\eta}^{-1} \| | \Pi_{\eta}(u^{\eta *} + \delta ) - \Pi_{\eta}(u^{\eta *}) |
\end{equation*}
Since $\| \partial \Pi_{\eta}(u^{\eta}) / \partial u^{\eta} \| \leq 1$, the inequality above is transformed into
\begin{equation*}
    | \Gamma_{\eta}(\delta) - \delta | \leq \| I_{3} - \Lambda_{\eta}G_{\eta}\bar{G}_{\eta}^{-1}\Lambda_{\eta}^{-1} \| | \delta |.
\end{equation*}
Based on (\ref{eq: WP - position and angular velocity values}) and the relation $\dot{R}_{I B} = R_{I B}[\omega^{B}_{I B}]$, if $\theta \in (-\frac{\pi}{2}, \frac{\pi}{2})$, $M^{*}_{\eta} \delequal Q^{-\top}M_{\eta}Q^{-1}$ is organized as follows.
\begin{equation} \label{eq: analyzed form of M_eta}
\begin{split}
    M^{*}_{\eta} &= m_b [R_{I b}\hat{p}_{b E}]^{\top}[R_{I b}\hat{p}_{b E}] + m_1 [R_{I1}\hat{p}^b_{1E}]^{\top}\\&\times [R_{I1}\hat{p}^b_{1E}] + m_2 [R_{I2}\hat{p}^b_{2E}]^{\top} [R_{I2}\hat{p}^b_{2E}]+ J_b \\ &+ R_{b 1}^{\top} J_1 R_{b 1} + R_{b 2}^{\top} J_2 R_{b 2} + R_{b E}^{\top} J_E R_{b E}
    \end{split}
\end{equation}
If we set $\bar{J}_b < J_b$, we can formulate the inequality below by substituting (\ref{eq: definition of F_eta and G_eta}) and $\Lambda_{\eta} = \bar{J}_b^{(1/2)}Q$ since $J_b < M^{*}_{\eta}$. 
\begin{equation}
    \| I_{3} - \Lambda_{\eta}G_{\eta}\bar{G}_{\eta}^{-1}\Lambda_{\eta}^{-1} \| = \| I_3 - \bar{J}_b^{\frac{1}{2}} M_{\eta}^{* -1}\bar{J}_b^{\frac{1}{2}} \| < 1.
\end{equation}
Thus, $| \Gamma_{\eta}(\delta) - \delta |$ turns out to be smaller than $| \delta |$ and there exists $\kappa \in (0, 1)$ such that $\Gamma_{\eta}(\delta)$ belongs to the sector [1 - $\kappa$, 1 + $\kappa$]. Accordingly, it is shown that the only solution of $\Gamma_{\eta}(\delta) = 0_{3 1}$ is $\delta = 0_{3 1}$. $\hfill \blacksquare$
\end{proof}

To prove that the fast dynamic system in Lemma \ref{lemma for standard singular form} is exponentially stable, we derive 1st order differential equations for $\tilde{\xi}^{\eta} \delequal \xi^{\eta}-\xi^{\eta *}$ and $\tilde{\zeta}^{\eta} \delequal \zeta^{\eta}-\zeta^{\eta *}$ as
\begin{equation} \label{eq: error dynamics of xi and zeta}
    \begin{split}
         \epsilon_{\eta}\dot{\tilde{\xi}}^{\eta} &= A^{\eta}_{\xi}\tilde{\xi}^{\eta} - \epsilon_{\eta} B^{\eta}_2 \{ F_{\eta} + G_{\eta} \tau_{b 0} + G_{\eta}\bar{G}_{\eta}^{-1}\Lambda_{\eta}^{-1} \\ & \qquad \qquad \qquad \times \Pi_{\eta}(\zeta^{\eta *}_{[1]} + \zeta^{\eta}_{[1]} + \bar{F}_{\eta}) + \Delta_{\eta} \} \\
         \epsilon_{\eta}\dot{\tilde{\zeta}}^{\eta} &= A^{\eta}_{\tilde{\zeta}}\tilde{\zeta}^{\eta} - B^{\eta}_2 a^{\eta}_0 \Gamma_{\eta}(\tilde{\zeta}^{\eta}_{[1]}) - \epsilon_{\eta}B^{\eta}_{1}\dot{\zeta}^{\eta *}_{[1]}
    \end{split}
\end{equation}
where $\tilde{\zeta}^{\eta}_{[1]} \delequal \zeta^{\eta}_{[1]} - \zeta^{\eta *}_{[1]}$, $A^{\eta}_{\zeta} \delequal$ \verb|blkdiag|$\{ A^{\eta}_1, A^{\eta}_2, A^{\eta}_3 \}$, $B^{\eta}_1 \delequal I_3 \otimes [1 \ 0]^{\top}$, $B^{\eta}_2 \delequal I_3 \otimes [0 \ 1]^{\top}$, $a^{\eta}_0 \delequal$ \verb|diag|$\{ a^{\eta}_{1, 0}, a^{\eta}_{2, 0}, a^{\eta}_{3, 0} \}$ and $A^{\eta}_{\tilde{\zeta}} \delequal$ \verb|blkdiag|$\{ A^{\eta}_{\tilde{\zeta}, 1}, A^{\eta}_{\tilde{\zeta}, 2}, A^{\eta}_{\tilde{\zeta}, 3} \}$ with $ A^{\eta}_{\tilde{\zeta}, i} \delequal
    \begin{bmatrix}
        0 & 1 \\
        0 & -a^{\eta}_{i, 1}
    \end{bmatrix}$.
\begin{remark} \label{remark: boundness of G_eta and Gbar_eta}
Since $\| \Lambda_{\eta}G_{\eta}\bar{G}^{-1}_{\eta}\Lambda_{\eta}^{-1} - I_3 \| \leq \kappa < 1$, the given inequality
\begin{multline*}
    | \Lambda_{\eta}G_{\eta}\bar{G}^{-1}_{\eta}\Lambda_{\eta}^{-1} v |^2 \\< (\Lambda_{\eta}G_{\eta}\bar{G}^{-1}_{\eta}\Lambda_{\eta}^{-1} v)^{\top} v + v^{\top} (\Lambda_{\eta}G_{\eta}\bar{G}^{-1}_{\eta}\Lambda_{\eta}^{-1} v)     
\end{multline*}
holds $\forall v \neq 0_{3 1}$. If there exists $\bar{v} \neq 0_{3 1}$ satisfying $\Lambda_{\eta}G_{\eta}\bar{G}^{-1}_{\eta}\Lambda_{\eta}^{-1} \bar{v} = 0_{3 1}$, it becomes contrary to the inequality above. Therefore, $\Lambda_{\eta}G_{\eta}\bar{G}^{-1}_{\eta}\Lambda_{\eta}^{-1}$ is an invertible matrix. Accordingly, since $G_{\eta}$ is not singular, $\| \bar{G}^{-1}_{\eta} \|$ is bounded. Moreover, since $\| \dot{G}^{-1}_{\eta} \| \leq \| G^{-1}_{\eta} \| \| \dot{G}_{\eta} \| \| G^{-1}_{\eta} \| $ holds and $(\bar{G}^{-1}_{\eta}, \dot{\bar{G}}^{-1}_{\eta})$ are bounded in $\mathcal{U}_{\eta}$ based on the relationship $\bar{G}_{\eta} = \bar{J}_{b}^{-1}Q^{T}$, $\| \dot{G}_{\eta}\|$ is also bounded.
\end{remark}
\begin{remark} \label{remark: remaining in U_eta}
Since the initial condition in Theorem \ref{thm: stability analysis on wire-pulling mode} is $[\bar{\eta}^N(0)^{\top} \ \dot{\bar{\eta}}^N(0)^{\top} \ \bar{\tau}^N_{b 0}(0)^{\top}]^{\top} = [\eta(0)^{\top} \ \dot{\eta}(0)^{\top} \ \tau_{b 0}(0)^{\top}]^{\top}$ and $\eta$ is bounded due to Remark \ref{remark: C^2 and boundedness of F,G} and the definition of $\Pi_{\eta}$, there exists $T_1 > 0$ such that $[\eta(t)^{\top} \tau_{b 0}^{\top}]^{\top}$ remains in $\mathcal{U}_{\eta}$ for $t \in [0 \ T_1]$. Also, there exists $T_2>0$ such that $| [\eta(0)^{\top} \ \dot{\eta}(0)^{\top} \ \tau_{b 0}(0)^{\top}]^{\top} - [\bar{\eta}^N(0)^{\top} \ \dot{\bar{\eta}}^N(0)^{\top} \ \bar{\tau}^N_{b 0}(0)^{\top}]^{\top} | \leq \sigma/2$ for $t \in [0 \ T_2]$.
\end{remark}

\begin{lemma} \label{lemma: equation for exponential stability of fast dynamics}
For $t_f \delequal$ min$\{ T_1, T_2 \}$, there exists $\epsilon^*_{\eta}$ such that the solutions of (\ref{eq: error dynamics of xi and zeta}) initiated from any $\xi(0)$ and $\zeta(0)$ satisfies
\begin{multline} \label{eq: equation for exponential stability of fast dynamics}
    | [\tilde{\xi}^{\eta}(t)^{\top} \ \tilde{\zeta}^{\eta}(t)^{\top}]^{\top} | \\ \leq \lambda_1 e^{-\lambda_2 (t/\epsilon_{\eta})} | [\tilde{\xi}^{\eta}(0)^{\top} \ \tilde{\zeta}^{\eta}(0)^{\top}]^{\top} | + \Omega(\epsilon_{\eta})
\end{multline}
for $\forall$ $t$ $\in [0 \ t_f]$ and $0 < \epsilon_{\eta} \leq \epsilon^{*}_{\eta}$, with some positive constants $\lambda_1$ and $\lambda_2$, and a class-$\mathcal{K}$ function $\Omega$.
\end{lemma}
\begin{proof}
The proof of Lemma \ref{lemma: equation for exponential stability of fast dynamics} is presented in \cite{lee2020aerial_2}. $\hfill \blacksquare$
\end{proof}

From the quasi-steady state results in (\ref{eq: quasi-states of xi and zeta}) and (\ref{eq: solution for ustar}), the reduced model of (\ref{eq: actual model of WP mode}) is derived as
\begin{equation}
    \begin{split}
        \ddot{\eta} &= F_{\eta} + G_{\eta}(\tau_{b 0}+(\Lambda_{\eta}\bar{G}_{\eta})^{-1}\Pi(u^{\eta *})) + \Delta_{\eta} \\
        &= F_{\eta} + G_{\eta}\tau_{b 0}+ G_{\eta}\bar{G}_{\eta}^{-1}u^{\eta *} + \Delta_{\eta} \\
        &= \bar{F}_{\eta} + \bar{G}_{\eta}\tau_{b 0}
    \end{split}
\end{equation}
which is exactly identical to the nominal model of $WP$ mode (\ref{eq: nominal dynamics of WP mode}). From Remark \ref{remark: remaining in U_eta} and Lemma \ref{lemma: equation for exponential stability of fast dynamics}, $| [\eta(0)^{\top} \ \dot{\eta}(0)^{\top} \ \tau_{b 0}(0)^{\top}]^{\top} - [\bar{\eta}^N(0)^{\top} \ \bar{\dot{\eta}}^N(0)^{\top} \ \bar{\tau}^N_{b 0}(0)^{\top}]^{\top} | \leq \sigma/2$ for $t \in [0 \ T_2]$ holds and $| [\tilde{\xi}^{\eta}(t_f)^{\top} \ \tilde{\zeta}^{\eta}(t_f)^{\top}] | \rightarrow 0$ as $\epsilon_{\eta} \rightarrow 0$. From those aspects, we can show the result of Theorem \ref{thm: stability analysis on wire-pulling mode} by applying Tikhonov's theorem for the infinite time interval presented in \cite[p.456]{khalil2002nonlinear}.

\subsubsection{ST and FF modes}
Due to Theorem 1 in \cite{lee2020aerial_2} and \cite{kim2017robust}, for a given $\sigma > 0$, the inequality $| [\chi(t)^{\top} \ \dot{\chi}(t)^{\top} \ u_0(t)^{\top}]^{\top} - [\bar{\chi}^N(t)^{\top} \ \dot{\bar{\chi}}^N(t)^{\top} \ u_0^{N}(t)^{\top}]^{\top} | \leq \sigma$ holds in $\forall t > 0$ if the condition $[\chi(0)^{\top} \ \dot{\chi}(0)^{\top} \ u_0(0)^{\top}]^{\top} = [\bar{\chi}^N(0)^{\top} \ \dot{\bar{\chi}}^N(0)^{\top} \ u_0^{N}(0)^{\top}]$ is satisfied.

\subsection{Stability and Robustness Analysis for the Entire Operation} 
It is shown that the union of $\bar{v}^N_{WP}$, $\bar{v}^N_{ST}$ and $\bar{v}^N_{FF}$ is a $\sigma$-robust $WP \mapsto ST$ and $ST \mapsto FF$ transition maneuver. Also, with the condition $\sigma \leq \delta_{\eta}$, $\eta$ in $ST$ mode unconditionally reaches $\mathcal{G}\{(ST, FF)\}$ in a finite time since $\| [\eta^{\top} \tau_b^{\top}]^{\top} - [\bar{\eta}^N \tau_{b, 0}^{\top}]^{\top} \| \leq \| v_{ST} - \bar{v}^N_{ST} \| \leq \sigma \leq \delta_{\eta}$ is proved to be valid in $ST$ mode. Subsequently, it is also proved that the actual value of $x_r$ and $x_q$ throughout the whole three modes satisfies the inequalities $\|[x_r^{\top} u_{f}^{\top}]^{\top} - \bar{v}^N_{WP}\| \leq \sigma$ and $\|[x_q^{\top} u_{f}^{\top}]^{\top} - \bar{v}^N_{ST \ \textrm{or} \ FF}\| \leq \sigma$ when the initial values of the solutions from the nominal and actual models are identical. Therefore, the stability and robustness of the whole plug-pulling task with the trajectories (\ref{eq: trajectory (WP)})-(\ref{eq: trajectory (FF)}) and the proposed control structure are guaranteed.

\section{EXPERIMENTAL RESULT \& DISCUSSION}
\subsection{Experimental Setup}
The experimental setup for this study consists of three parts: a hexacopter, a robotic arm, and a frame for the plug-socket system. We assembled the hexacopter with the off-the-shelf frame DJI F550, six KDE2315XF-967 motors with corresponding KDEXF-UAS35 35A+ electronic speed controllers and 9-inches T-Motor polymer rotors, two Turnigy LiPo batteries for power supplement, and Intel NUC for computing. It executes the proposed DOB controller, the Kalman filter for estimating the value of $\ddot{\gamma}$ and the navigation algorithm with OptiTrack on Robot Operating System (ROS) in Ubuntu 18.04, and a flight controller Pixhawk 4, all onboard. The robotic arm is comprised of ROBOTIS dynamixel XH430 and XM540 servo motors. We manufacture a stand for a 110 V socket and firmly attach it to the wooden plate which weighs about 11  kg as in Fig. \ref{fig: socket_frame}.
\begin{figure}[t]
\begin{subfigure}{0.24\textwidth} \centering
\includegraphics[width=0.57\linewidth]{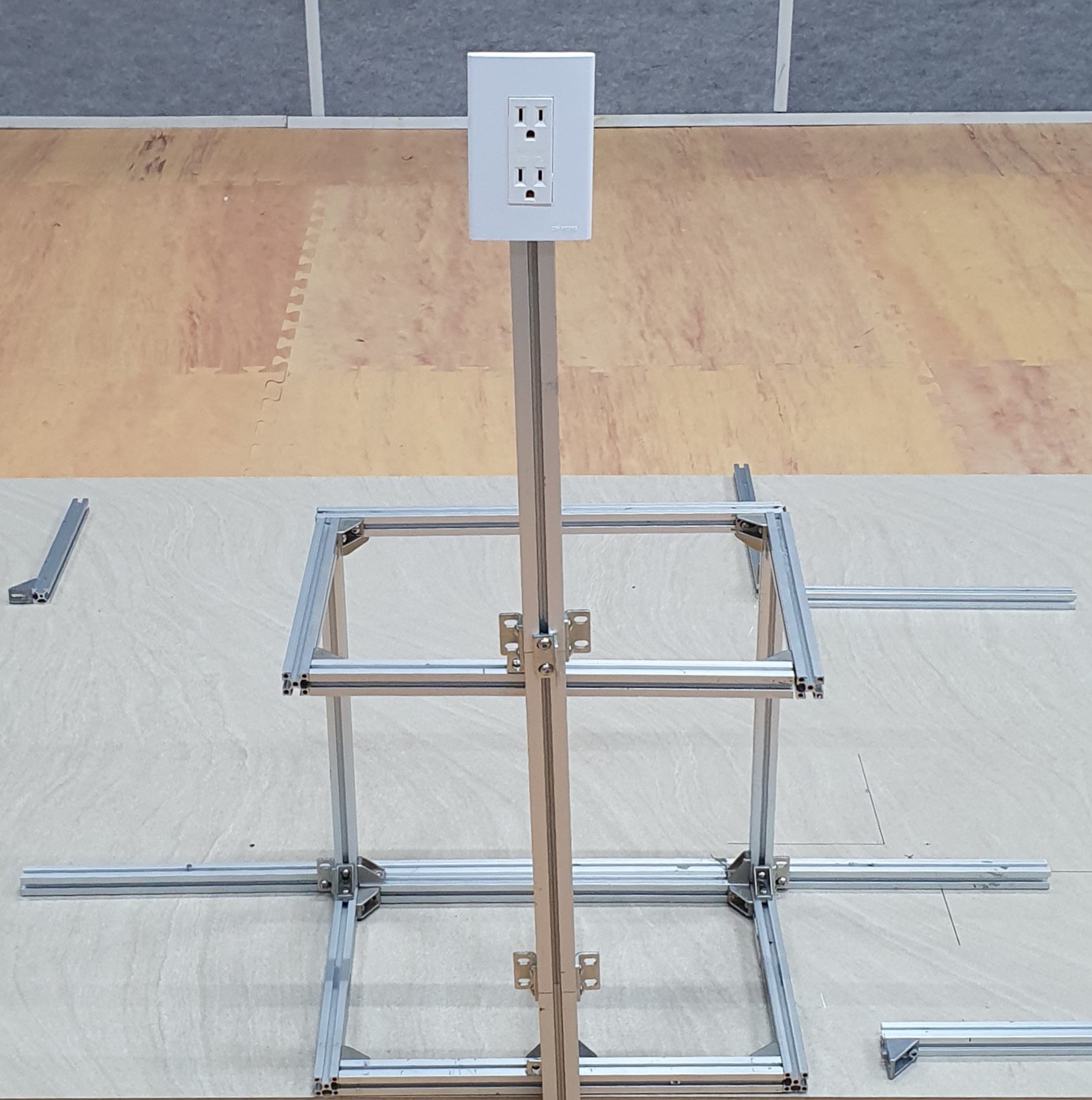} 
\caption{Overall view} 
\label{fig: socket_frame_front_view}
\end{subfigure}
\begin{subfigure}{0.24\textwidth} \centering
\includegraphics[width=0.76\linewidth]{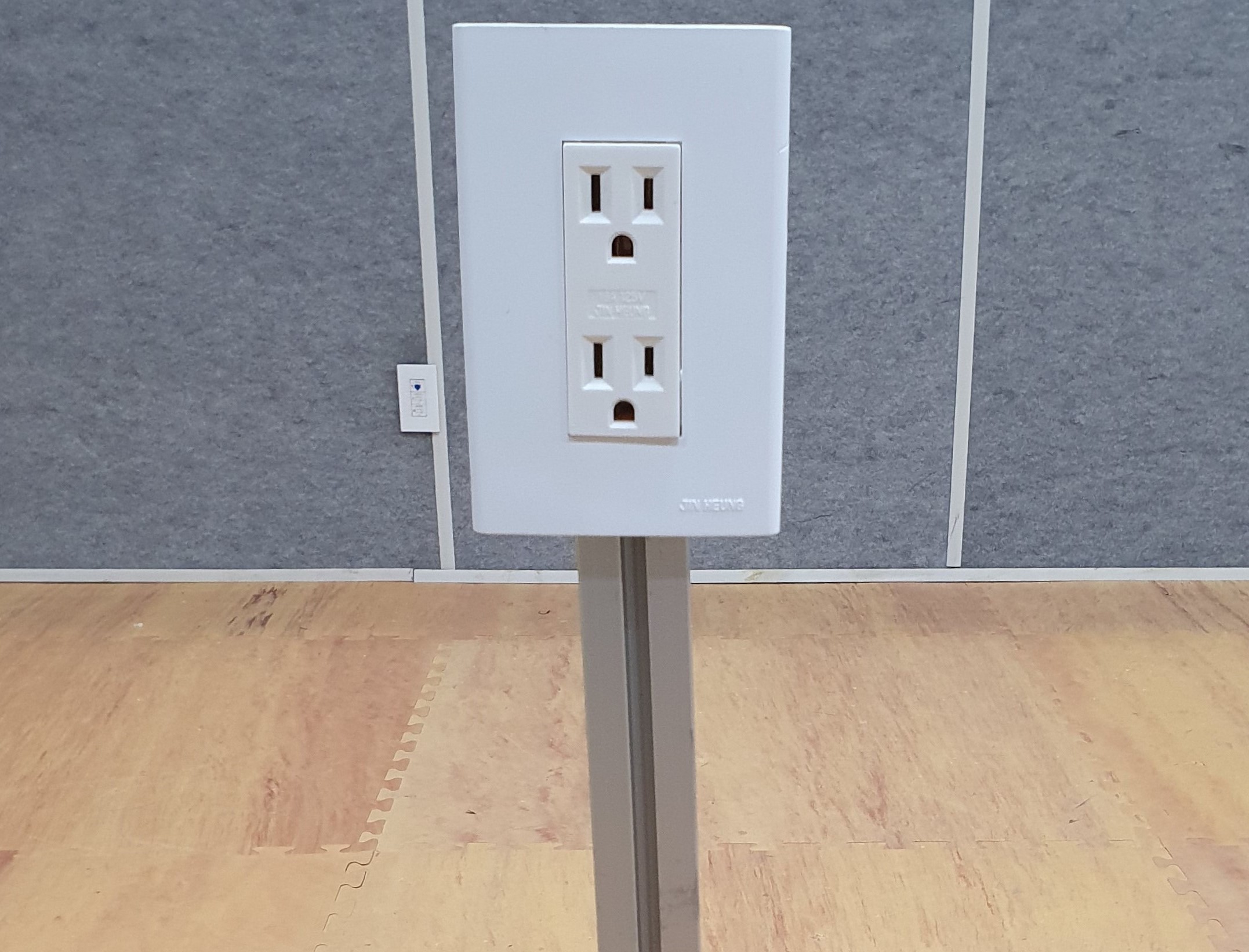}
\caption{Detail view}
\label{fig: socket_frame_detail_view}
\end{subfigure}
\caption{The stand for a 110 V socket used for the plug-pulling experiment}
\label{fig: socket_frame}
\vspace{-0.4cm}
\end{figure}

As depicted in Fig. \ref{fig: wire-pulling aerial manipulator}, the plug is connected to a wire held by the end-effector. Since this is different from the perching model described in Fig. \ref{fig: perching aerial manipulator with frames}, we modify some strategies for the trajectory generation and the controller design to maintain the end-effector's position fixed. For this, we turn on the position controller in $\hat{j}_I$ and $\hat{k}_I$ directions as in Fig. \ref{fig: control_frame}. 
\begin{figure}[t]
\centering
\includegraphics[width = 0.39\textwidth]{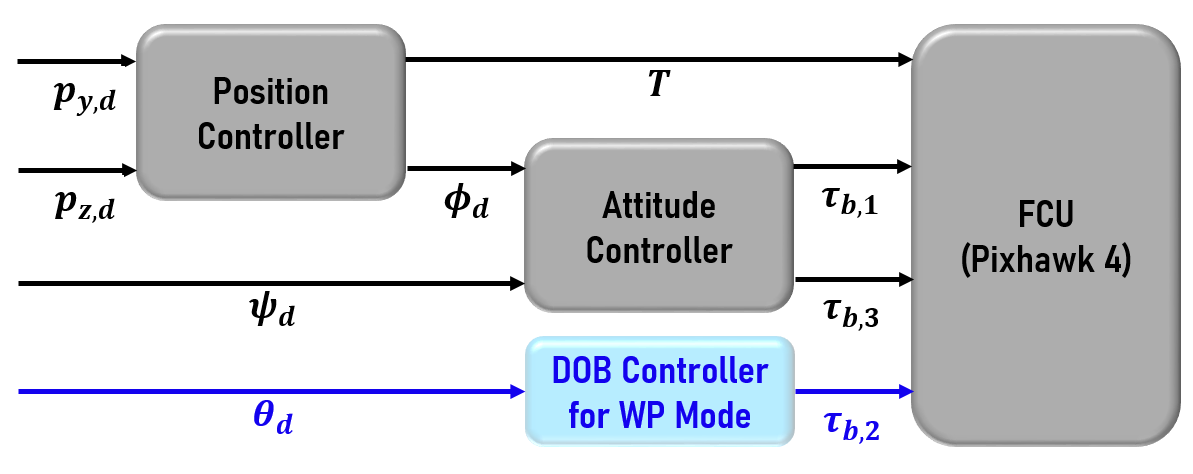}
\caption{Strategy for generating control input $u$ in WP mode while conducting the actual experiment. We use the controller presented in \cite{lee2020aerial_2} for the position and attitude control.} \label{fig: control_frame}
\vspace{-0.4cm}
\end{figure}

We set $\theta_m$ as $20$ deg and the stabilizing time $t_{d, ST} - t_{0, ST}$ as $0.08$ sec. Also, the threshold for the mode change from ST to FF, $\delta_{\eta}$, is set as $5$ deg.

The scenario of the experiment is as follows. At first, the aerial manipulator takes off until the end-effector aligns with the plug attached to the socket. Then, the vehicle flies in $-\hat{i}_I$ direction to tighten the wire. When the wire becomes tight enough, we send a command to begin the WP mode and the vehicle gradually tilts its body. By the time when the plug is separated from the socket, we manually switch to ST mode and the vehicle automatically turns into FF mode when the Euler angles become reasonably small. Finally, the aerial manipulator returns to the original position and maintains the hovering state.

\subsection{Experimental Results}
\begin{figure}[t]
\centering
\includegraphics[width = 0.45\textwidth]{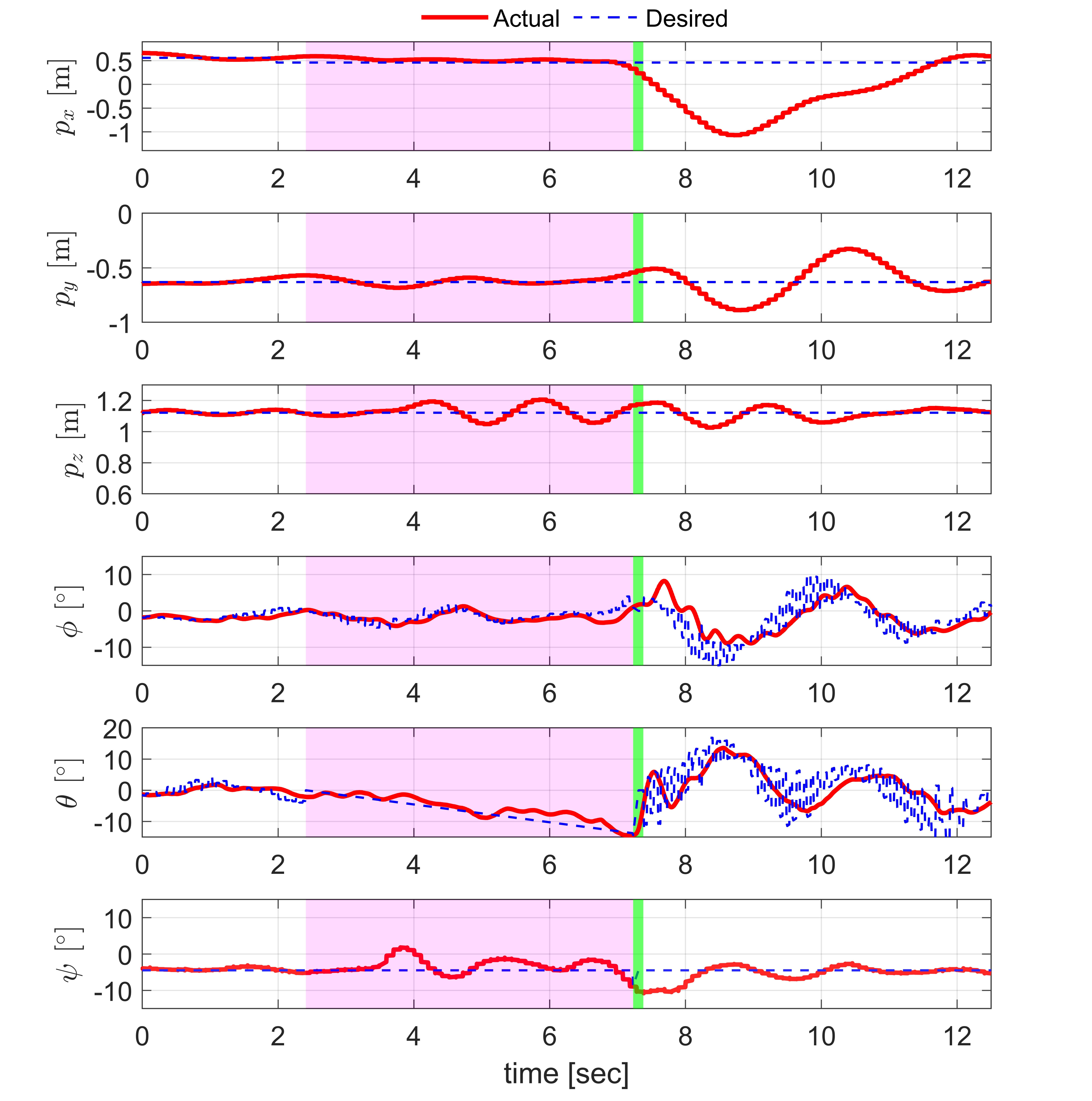}
\caption{History of the states of multirotor through the entire operative modes. A white region expresses the FF mode, a magenta region expresses the WP mode and a green region expresses the ST mode. The dashed blue line describes the desired values of the states. The red line represents the measured values of the states.} \label{fig: T_s_80ms_w_DOB_final}
\vspace{-0.6cm}
\end{figure}
Fig. \ref{fig: T_s_80ms_w_DOB_final} shows that the measured value of $\theta$ adequately follows $\theta_d$ until it reaches $-14.4$ deg. From this result, we can confirm that the control structure introduced in (\ref{eq: nominal dynamics of WP mode}) is valid for the given task. Also, from the plots of $p_x$, $p_y$ and $p_z$ in Fig. \ref{fig: T_s_80ms_w_DOB_final}, we can observe that the aerial manipulator recovers its original position at about 12 sec. This result shows that the hybrid automata of the aerial manipulator conducting the plug-pulling are stable and robust to the sudden change of interaction force. 

On the other hand, as can be seen in the attached video, if a standard PID controller is employed for the wire-pulling, $\theta$ cannot adequately follow the desired value $\theta_d$. Accordingly, after the plug is separated from the socket, the vehicle fails to maintain its stability and crashes to the floor.

\section{CONCLUSION}
This paper presents an aerial manipulator consisting of a multirotor and a 2-DOF robotic arm pulling a plug out of a socket. To demonstrate aerial plug-pulling, the concept of hybrid automata is used to divide the mission into three operative modes of wire-pulling, stabilizing, and free-flight. The strategy for trajectory generation and design of DOB controllers based on the dynamical models of each operative mode are presented. Then, we prove the overall stability and robustness of the plug-pulling and validate them through the actual experiment. As a result, we confirm that the pitch angle can robustly track {the desired pitch trajectory by our proposed DOB controller}, and the overall sequence of the plug-pulling task is executed without destabilization. For future work, autonomous grasping of a plug with an aerial manipulator can be included.






\bibliographystyle{IEEEtran}
\bibliography{myreference}

\end{document}